\documentclass[twoside]{article}

\usepackage[accepted]{aistats2023}
\usepackage{amsthm}

\usepackage[round]{natbib}

\usepackage{graphicx} 

\usepackage{amsmath}
\usepackage{amssymb}

\newcommand{\eps}{\epsilon}

\newcommand{\mc}{\mathcal}
\newcommand{\mb}{\mathbb}

\newcommand{\T}{\intercal}

\newtheorem{assumption}{Assumption}
\newtheorem{proposition}{Proposition}

\newtheorem{theorem}{Theorem}
\newtheorem{corollary}{Corollary}

\def\ddefloop#1{\ifx\ddefloop#1\else\ddef{#1}\expandafter\ddefloop\fi}
\def\ddef#1{\expandafter\def\csname bb#1\endcsname{\ensuremath{\mathbb{#1}}}}
\ddefloop ABCDEFGHIJKLMNOPQRSTUVWXYZ\ddefloop
\def\ddef#1{\expandafter\def\csname bf#1\endcsname{\ensuremath{\mathbf{#1}}}}
\ddefloop ABCDEFGHIJKLMNOPQRSTUVWXYZabcdefghijklmnopqrstuvwxyz\ddefloop
\def\ddef#1{\expandafter\def\csname bs#1\endcsname{\ensuremath{\boldsymbol{#1}}}}
\ddefloop ABCDEFGHIJKLMNOPQRSTUVWXYZabcdefghijklmnopqrstuvwxyz\ddefloop
\def\ddef#1{\expandafter\def\csname sf#1\endcsname{\ensuremath{\mathsf{#1}}}}
\ddefloop ABCDEFGHIJKLMNOPQRSTUVWXYZ\ddefloop
\def\ddef#1{\expandafter\def\csname c#1\endcsname{\ensuremath{\mathcal{#1}}}}
\ddefloop ABCDEFGHIJKLMNOPQRSTUVWXYZ\ddefloop

\DeclareMathOperator*{\argmin}{argmin}

\newcommand{\xnom}{x_{PR}}
\newcommand{\dxnom}{\dot x_{PR}}
\newcommand{\xpert}{x_{RR}}
\newcommand{\dxpert}{\dot x_{RR}}

\newcommand{\fnom}{f_{PR}}
\newcommand{\fpert}{f_{RR}}

\usepackage{enumitem}

\begin{document}

%

%

\twocolumn[

\aistatstitle{Approximate Regions of Attraction in Learning with Decision-Dependent Distributions}

\aistatsauthor{ Roy Dong \And Heling Zhang \And  Lillian J. Ratliff }

\aistatsaddress{ UIUC \And  UIUC \And UW } ]

\begin{abstract}
  As data-driven methods are deployed in real-world settings, the processes that generate the observed data will often react to the decisions of the learner. For example, a data source may have some incentive for the algorithm to provide a particular label (e.g. approve a bank loan), and manipulate their features accordingly. Work in strategic classification and decision-dependent distributions seeks to characterize the closed-loop behavior of deploying learning algorithms by explicitly considering the effect of the classifier on the underlying data distribution. More recently, works in performative prediction seek to classify the closed-loop behavior by considering general properties of the mapping from classifier to data distribution, rather than an explicit form. Building on this notion, we analyze repeated risk minimization as the perturbed trajectories of the gradient flows of performative risk minimization. We consider the case where there may be multiple local minimizers of performative risk, motivated by situations where the initial conditions may have significant impact on the long-term behavior of the system. We provide sufficient conditions to characterize the region of attraction for the various equilibria in this settings. Additionally, we introduce the notion of performative alignment, which provides a geometric condition on the convergence of repeated risk minimization to performative risk minimizers.
\end{abstract}


\section{INTRODUCTION}
\label{sec:intro}

Data-driven methods are growing increasingly popular in practice. Most classical machine learning and statistical methods view the underlying process which generates the data as fixed: the study is primarily focused on the mapping from data distributions to classifier. However, it is important to consider the effects in the other direction as well: how does the classifier chosen by a learner change the data distribution the learner sees? In particular, how do we close the loop around machine learning deployments in practice?

These closed loop effects can arise in many real world settings. One instance is strategic classification: whenever a data source has a stake in which label a classifier applies to it, they will seek cost-effective ways to manipulate their data to earn the desired label. For example, credit scoring classifiers are heavily guarded for fear of the potential for gaming~\citep{Hardt:2016we}. 
Alternatively, deployments of the classifier can both skew future datasets and also have causal influences over the real-world processes at play. For example, a classifier that predicts crime recidivism influences the opportunities available to individuals~\citep{Dressel:2018uf}.

Formally, we consider this problem in the framework introduced in~\citet{Perdomo:2020tz}. Let $\ell(z,x)$ denote the loss when the learner's decision is $x$ (e.g. $x$ can be the parameters of the chosen classifier) and the data has realized value $z$. Furthermore, let $\mc{D}(x)$ denote the data distribution when the learner's decision is $x$. In this framework, the performative risk is given by:
\begin{equation}
\label{eq:perf_risk}
PR(x) = \mb{E}_{Z \sim \mc{D}(x)} [\ell(Z,x)]
\end{equation}
Whereas classical machine learning results treat the distribution $Z \sim \mc{D}$ as fixed, the performative prediction framework models the decision-dependent distribution as a mapping $\mc{D}(\cdot)$. However, in many real world-deployments, this decision-dependent distribution shift may not be explicitly included in the learner's updates. This leads to algorithms based on inexact repeated minimization. Define the decoupled performative risk as:
\begin{equation}
\label{eq:decoupled_perf_risk}
R(x_1,x_2) = \mb{E}_{Z \sim \mc{D}(x_2)} [\ell(Z,x_1)]
\end{equation}
The decoupled performative risk $R(x_1,x_2)$ separates the two ways that the decision variable $x$ affects the performative risk. Through the $x_1$ argument, $x$ affects the classification error; through the $x_2$ argument, $x$ causes a decision-dependent distribution shift. 
Thus, when the decision-dependent distribution shift is not accounted for, the repeated gradient descent method yields the following update rule:
\begin{equation}
\label{eq:rep_motiv_eq}
x_{k+1} = x_k - \alpha_k (\nabla_{x_1} R(x_k,x_k) + \eta_k)
\end{equation}
Here, $(\eta_k)_k$ is some zero-mean noise process. 
Note that the gradient is evaluated only with respect to the first argument, i.e. the updates are based only on the effect of $x$ on the loss function, and ignore the distribution shift caused by $x$. In other words, the learner draws several observations from the distribution $\mc{D}(x_k)$, and, treating this distribution as fixed, updates their model parameters $x_{k+1}$ based on stochastic gradient descent: they are descending the gradient of the cost function $y \mapsto R(y,x_k)$.

In this paper, we shall analyze the steady-state behavior of the continuous-time flows corresponding to Equation~\eqref{eq:rep_motiv_eq}:
\begin{equation}
\label{eq:rgd_flow1}
\dot x = - \nabla_{x_1}R(x,x)
\end{equation}
The connections between the flow of Equation~\eqref{eq:rgd_flow1} and the repeated gradient descent method in Equation~\eqref{eq:rep_motiv_eq} can be drawn using results in stochastic approximation, i.e. the latter can be seen as a noisy forward Euler discretization of the former. For more details, we refer the reader to~\citet{Borkar:2008ts}.

In particular, we focus on settings where there may be multiple local equilibria, and classify their regions of attraction for these equilibria. In many settings of interest, there may be multiple steady-state outcomes, and it is of interest to determine which outcome will be chosen by the dynamics in Equation~\eqref{eq:rep_motiv_eq}. Our results allow us to characterize which regions of the parameter space will converge to which equilibria. 
We discuss this example in greater formal detail in Section~\ref{sec:example}.

Our main theoretical results can be informally summarized as follows. Theorem 1 states that trajectories of inexact repeated risk minimization will converge exponentially fast to a neighborhood of local performative risk minimizers, and stay in this neighborhood for all future time. It also provides a sufficient condition to under-approximate the regions of attraction for each local performative risk minimizer. In the special case of vanishing perturbations, these trajectories will converge to the minimizers themselves. As a corollary, this implies that performatively stable points will be near performatively optimal points, which was first observed in~\citet{Perdomo:2020tz} under a different set of conditions. 
We note that Theorem 1 requires conditions on the curvature of the performative risk: the sublevel sets $\{ x : PR(x) \le c \}$ must grow in a precise fashion, such that upper and lower bounds on the performative risk $PR(x)$ imply upper and lower bounds on the norm of the argument $x$.
Furthermore, the gradient of the performative risk must not vary too wildly around nearby points. This is formalized in Assumption~\ref{ass:exist_V}. 
Theorem 2 states a geometric condition on the performative perturbation which ensures that trajectories of repeated risk minimization will converge to local performative risk minimizers, intuitively based on the idea that the perturbation does not push against convergence. This result does not require the strong curvature assumptions of Theorem 1.

These results allow us to identify the regions of attraction for various steady-state outcomes. As observed in~\citet{Miller:2021te}, these various outcomes can be interpreted as different echo chambers: essentially the decision variable $x$ can act as a sort of self-fulfilling prophecy.\footnote{It is worth noting that we take a slightly different interpretation of an `echo chamber' in this paper. In~\citet{Miller:2021te}, the echo chambers are defined as performatively stable points. In this paper, we consider the regions near each locally performatively optimal point as an echo chamber. As we will discuss in Section~\ref{sec:example}, we are interested in settings where there may be many local performative risk minimizers that attract learning methods depending on initialization.} In settings with multiple echo chambers, we consider the question of which echo chamber will come to dominate, based on the initialization of the learner. 

The rest of the paper is organized as follows. In Section~\ref{sec:background}, we discuss the related literature. In Section~\ref{sec:model}, we introduce the problem statement and the mathematical concepts used for our results, and provide motivating examples in Section~\ref{sec:example}. In Section~\ref{sec:analysis_prm}, we analyze the gradient flow associated with performative risk minimization, and in Section~\ref{sec:analysis_RGD}, we analyze the flows associated with repeated risk minimization. We demonstrate numerical results in Appendix~\ref{sec:num_results}, and provide closing remarks in Section~\ref{sec:conclusion}.


\section{BACKGROUND}
\label{sec:background}

There has been a great deal of interest in studying decision-dependent distributions. In the context of operations research, this has been studied under either the name decision-dependent uncertainty or endogenous uncertainty. In~\citet{Jonsbraten:1998wc},~\citet{Jonsbraten:1998wk}, and~\citet{Goel:2004wb}, the authors considered oil field optimization, with a framework that captures how information revelation can be affected by one's decisions. In~\citet{Peeta:2010tb}, the authors consider infrastructure investment, and how investments can affect the future likelihood of disasters. For a taxonomy of the work in the operations research community, we refer the reader to~\citet{Hellemo:2018tf}.

Another form of decision-dependent distributions is strategic classification. In these works, the data source is seen as a utility-maximizing agent. The distribution shift resulting from the learner's decision is modeled by a best response function. In~\citet{Hardt:2016we} and~\citet{Bruckner:2011wy}, the authors formulate the problem as a Stackelberg game where the data source responds to the announced classifier. In~\citet{Dong:2018td}, the authors consider when the data source's preferences are hidden information and provide sufficient conditions for convexity of the overall strategic classification task. In~\citet{Akyol:2016wu}, the authors quantify the cost of strategic classification for the classifier. In~\citet{Milli:2019tf} and~\citet{Hu:2019wu}, the authors note that certain groups may be disproportionately affected as institutions incorporate methods to counter data sources gaming the classifier. In~\citet{Miller:2020vy}, the authors formulate strategic classification in a causal framework.

Most related to our work is recent efforts in performative prediction. This was introduced in~\citet{Perdomo:2020tz}. In this formulation, rather than explicitly modeling the form of the distribution shift, it proposes to analyze the decision-dependent distribution shift in terms of general properties of the $\mc{D}(\cdot)$ mapping, where $\mc{D}(x)$ is the distribution of the data when the learner's decision is $x$. In~\citet{Perdomo:2020tz}, the authors introduced the concepts related to performative prediction, demonstrated that neither the performatively stable nor performatively optimal points are subsets of each other, provided sufficient conditions for exact repeated risk minimization (defined as finding the exact minima with respect to $\mc{D}(x_k)$ at each time step) to converge, and provided conditions in which performatively stable points are near performatively optimal points. In~\citet{Mendler-Dunner:2020vd}, the authors analyze inexact repeated risk minimization (defined as an update step with respect to $\mc{D}(x_k)$ at each time step) from a stochastic optimization framework. In this paper, we build on the inexact repeated risk minimization framework. \citet{Miller:2021te} provided sufficient conditions for performative risk itself to be convex. \citet{Brown:2020wg} extended these results to settings where the distribution updates may have an internal state. 
In~\citet{Drusvyatskiy:2020wk}, the authors show that many inexact repeated risk minimization algorithms will also converge nicely, due to the way in which the performative perturbation decays near the solution. This shares many ideas with our work here, but we focus on the case where there may be multiple attractive equilibria, and generalize to settings where the perturbation itself may not vanish. 
In contrast to previous works which provide sufficient conditions to guarantee that an outcome is approached globally, we focus on understanding local regions of attraction for various outcomes.

This work draws on ideas from control theory; in particular, the analysis of gradient flows, Lyapunov functions, and perturbation analysis are the tools we use throughout. We refer the reader to~\citet{Hirsch:2012tx} and~\citet{Khalil:2001wj} as good references for these suite of tools.

Although our work still focuses on repeated risk minimization, it is worth noting that many other algorithms exist for learning with decision-dependent distributions. 
In \citet{Jag22}, the authors proposes the \textit{performative confidence bounds algorithm} which uses tools from the bandit literature to explore the distribution map and find a near-optimal solution. 
In \citet{Izzo21_1}, the authors proposes an algorithm called \textit{performative gradient descent} (PerfGD), which guarantees to find the performatively optimal point when $\cD(\cdot)$ satisfies certain parametric assumptions. Later in \citet{Izzo21_2}, the author extends the results to settings where the distribution updates have internal states. In the same year, \cite{Li21} presents state-dependent stochastic approximation (SA) algorithm that works in similar settings. Of course, this paragraph is not a exhaustive treatment of various algorithms in similar settings. 


\section{PERFORMATIVE PREDICTION, FLOWS, AND PERTURBATIONS}
\label{sec:model}

In this section, we introduce the mathematical concepts used throughout this paper. As previously mentioned, the framework used throughout this paper builds on the framework of performative prediction, introduced in~\citet{Perdomo:2020tz}.

In Section~\ref{sec:intro}, we have already defined the \textbf{performative risk} in Equation~\eqref{eq:perf_risk} and the \textbf{decoupled performative risk} in Equation~\eqref{eq:decoupled_perf_risk}. 
Furthermore, we say that $x$ is a \textbf{local performative risk minimizer} is $x$ is a local minima of $PR(\cdot)$. We say $x$ is \textbf{locally performatively stable} if $x$ is a local minima of $y \mapsto R(y,x)$. In general, neither imply the other~\citep{Perdomo:2020tz}.

Additionally, we consider the \textbf{performative risk minimizing (PRM) gradient flow}, defined by the following differential equation:
\begin{equation}\label{eq:prm_flow}
   \begin{aligned}
\dxnom &= - \nabla PR(\xnom) \\
&= - \nabla_{x_1} R(\xnom,\xnom) - \nabla_{x_2}R(\xnom,\xnom)\\
&=: \fnom(\xnom)
\end{aligned} 
\end{equation}
This vector field can be represented by the gradient of a function, which lends the flow to nice analysis. Under mild conditions, the trajectories of Equation~\eqref{eq:prm_flow} will converge to local minima of the performative risk.

However, as noted in Section~\ref{sec:intro}, many deployments of machine learning do not explicitly model the distribution shift, and, consequently, do not directly minimize the performative risk. We define the \textbf{repeated gradient descent (RGD) flow} as solutions to the differential equation:
\begin{equation}
\label{eq:RGD_flow}
\dxpert = -\nabla_{x_1}R(\xpert,\xpert) =: \fpert(\xpert)
\end{equation}
We define the \textbf{performative perturbation}:
\[
g(x) := \nabla_{x_2}R(x,x) = \fpert(x) - \fnom(x)
\]
In this paper, we view the PRM gradient flow as the \textit{nominal} dynamics, and the RGD flow as the \textit{perturbed} dynamics. The PRM gradient flow has nice properties arising from the fact it is a gradient flow, and, under certain conditions on the performative perturbation, we can prove properties about the RGD flow, which is the quantity of interest. In particular, we show ultimate bounds on the distance between the trajectories of RGD flow and the local performative risk minimizers. This also implies that under certain conditions on the performative risk, all performatively stable points are near performative risk minimizers, as was observed in~\citet{Perdomo:2020tz}.

Throughout this paper, we will be using tools from perturbation analysis in control theory. For a complete vector field $\dot x = f(x)$, let $\varphi_f(\cdot;x_0)$ denote the unique solution to the differential equation with initial condition $x(0) = x_0$. For a scalar-valued function $V$ and a vector field $f$, we can define the derivative along trajectories as $\mc{L}_fV(x) = \frac{\partial V}{\partial x}f(x)$. 
We say a point $x$ is an \textbf{equilibrium point} if $f(x) = 0$. An equilibrium point $x$ is \textbf{locally asymptotically stable} if there exists a neighborhood $U \ni x$ such that $\lim_{t \rightarrow \infty} \varphi_f(t;x') = x$ for all $x' \in U$. A set $A$ is \textbf{positively invariant} if for all $x_0 \in A$ and $t \ge 0$, we have $\varphi_f(t;x_0) \in A$. 
Additionally, given a set $A \subset \mb{R}^n$, we say two points $x$ and $y$ are \textbf{path-connected in $A$} if there exists a continuous function $\gamma : [0,1] \to A$ such that $\gamma(0) = x$ and $\gamma(1) = y$. This forms an equivalence relation defined on $A$, and each equivalence class is a \textbf{connected component of $A$}. Additionally, we will use $\mc{W}_1(\cdot)$ to denote the \textbf{Wasserstein distance}, also known as the earth mover's distance.


\subsection{EXAMPLES}
\label{sec:example}

Before we present our analysis of the PRM gradient flow and the RGD flow, we introduce some examples which motivate the study of performative risk in non-convex settings and multiple local equilibria. 

\subsubsection{Squared error loss and Bernoulli distributions}
\label{sec:simple_ex}

Consider the loss function $\ell(z,x) = \frac{1}{2} |z-x|^2$, where $x$ is a scalar. Furthermore, suppose that the decision-dependent distribution $\mc{D}(x)$ is simply $Z = 1$ with probability $p(x)$ and $Z = 0$ with probability $1 - p(x)$, for some function $p(\cdot)$. 
In this case, the decoupled performative risk is given by:
\begin{equation}
\label{eq:simple_dpr}
\begin{aligned}
R(x_1,x_2) &= p(x_2) \left[ \frac{1}{2} |1-x_1|^2 \right] + (1 - p(x_2)) \left[ \frac{1}{2} |x_1|^2 \right]\\
&= \frac{1}{2} [x_1^2 + p(x_2) (1 - 2x_1)]
\end{aligned}
\end{equation}
We will analyze this model in two ways. First, we will consider general $p(\cdot)$, and, fixing the loss function $\ell(\cdot)$, identify a class of decision-dependent distribution shifts $p(\cdot)$ which can still ensure convergence to performative risk minimizers, using Theorem~\ref{th:perf_align}. Second, we will consider a concrete example for $p(\cdot)$, and demonstrate how to apply Theorem~\ref{th:perturb1} to understand the regions of convergence.

As our concrete example of $p(\cdot)$, consider the following function as a candidate for $p(\cdot)$:
\begin{equation}
\label{eq:varphi_def}
\varphi(x) = 
\begin{cases}
\exp\left( 1 + \frac{-1}{1-(x-1)^2} \right) & \text{if } x \in (0,1) \\
0 & \text{if } x \le 0 \\
1 & \text{if } x \ge 1
\end{cases}
\end{equation}
This function is chosen because $\varphi(x) = 1$ for $x \ge 1$, $\varphi(x) = 0$ for $x \le 0$, and it is continuously differentiable. The derivative is:
\begin{equation}
\label{eq:dvarphi_def}
\varphi'(x) = 
\begin{cases}
\varphi(x) \frac{2(1-x)}{(1-(x-1)^2)^2} & \text{if } x \in (0,1) \\
0 & \text{otherwise}
\end{cases}
\end{equation}
$\varphi(\cdot)$ and its derivative is visualized in Figure~\ref{fig:varphi_plt}(a).  Since $p(0) = 0$ and $p(1) = 1$ for this choice of $p(\cdot)$, Equation~\eqref{eq:simple_dpr} directly implies that there are two performative risk minimizers: $x = 0$ and $x = 1$. Similarly, we can see that these points are performatively stable as well. 
The corresponding performative risk and gradients are visualized in Figure~\ref{fig:varphi_plt}(b)--(c).

\begin{figure*}[t!]
  \centering
\includegraphics[width=0.3\textwidth]{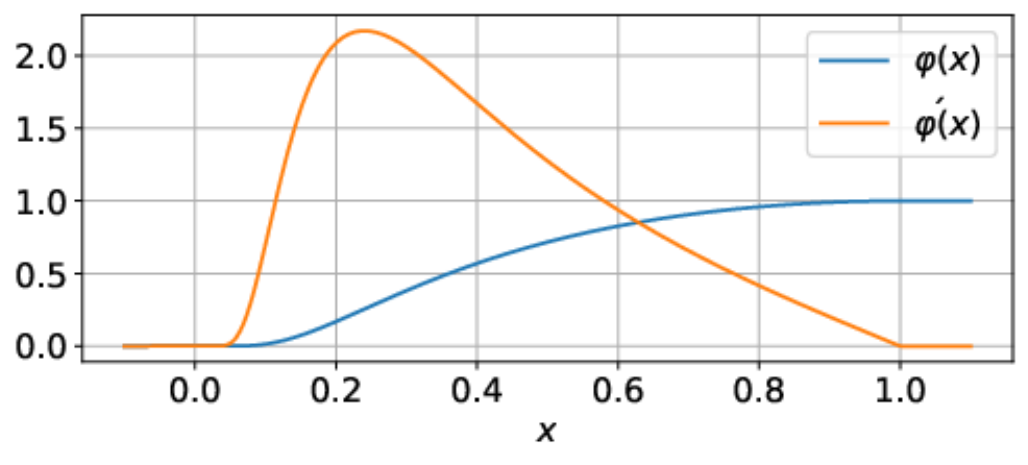}
\includegraphics[width=0.3\textwidth]{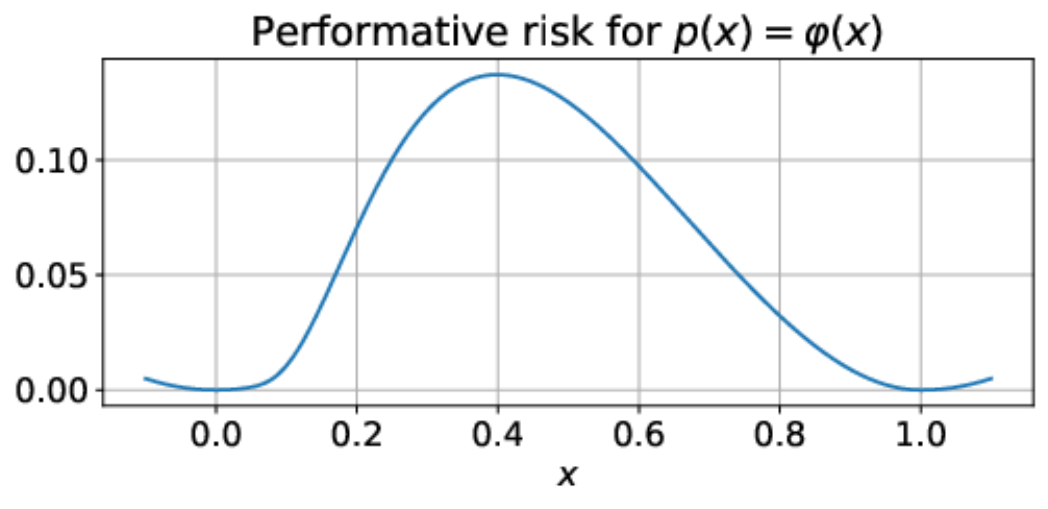}
\includegraphics[width=0.3\textwidth]{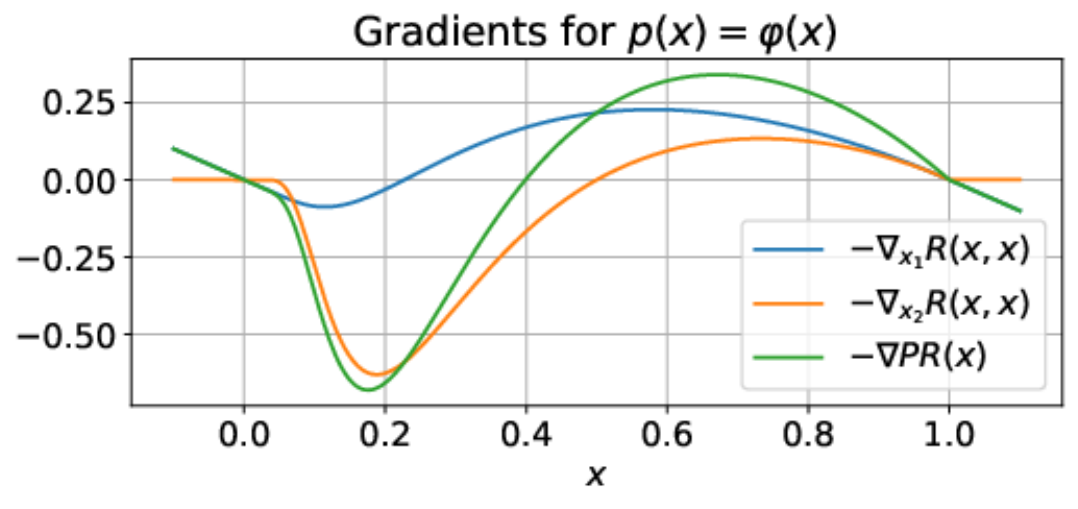}
  \caption{As an illustrative example, we consider a setting where with the squared error is used as the loss function, and the decision-dependent distribution shift modifies the parameters of a Bernoulli distribution, as discussed in Section~\ref{sec:simple_ex}. (a) A visualization of an example decision-dependent distribution shift $\varphi(x)$, as defined in Equation~\eqref{eq:varphi_def}, and its derivative, as derived in Equation~\eqref{eq:dvarphi_def}.  (b) The performative risk with $\Pr(Z = 1) = p(x) = \varphi(x)$. (c) The corresponding gradients for $p(x) = \varphi(x)$.}
  \label{fig:varphi_plt}
\end{figure*}

\subsubsection{Classification of adversarial agents}
\label{sec:adv_class}
As we've mentioned, when data-driven algorithms are deployed in real-world settings, it often caused a drift in the distribution of the data. One source of such drift is the behavior of adversarial agents. Here, we consider a simple classification problem of potentially adversarial agents.

Suppose that each agent are defined by a feature $z \in \bbR^d$ and a binary label $y \in \left\{ -1, 1 \right\}$, and
suppose that they are drawn i.i.d from a distribution $\cD_{Z, Y}$.
The task of the classifier is to correctly predict their
labels based on their features.
For simplicity, we further assume that
we are using a linear classifier defined by
\[
\begin{aligned}
    \hat{y}_{x}
    =
    sign(\langle x, z\rangle),
\end{aligned}
\]
where $x \in \bbR^d$ is the parameter of the classifier.
We also assume that the learner is using a logistic loss function:
\[
\begin{aligned}
    \ell(x, y, z)
    =
    \log(1 + \exp(-y\langle x, z\rangle)).
\end{aligned}
\]

So far, we've essentially described a ordinary binary classification problem with linear classifiers.
We further assume that once the classifier is deployed,
the agents will potentially alter their features to
induce false predictions.
Formally, we assume that if an agent is adversarial,
it will produce a fake feature $\hat{z}$ such that
\[
\begin{aligned}
    \hat{z}(x, y, z)
    \in
    \argmin_{z'}
    \left\{
        -y\langle x,z'\rangle
        +
        \frac{k}{2}
        \left\|z' - z\right\|_2^2
    \right\},
\end{aligned}
\]
and that agents are adversarial with probability
\[
\begin{aligned}
    p_{adv}(x, y, z)
    =
    e^{-\lambda_1 \left\|\hat{z}(x, y, z) - z\right\|_2^2}.
\end{aligned}
\]
Essentially, we are implicitly describing a `cost'
of adversarial behavior, which
that is proportional to the squared $l_2$ distance
between true and fake features. If the cost is too high,
an agent will be more likely to abandon adversarial
behavior, hence a lower $p_{adv}$.

This adversarial behavior can be thought of a distribution
shift caused by deploying our classifier, and the performative
risk is then given by
\begin{align*}
    PR(x)
    &=
    \bbE_{(\hat{Z}, Y) \sim \cD(x)}
    \left[ \ell(x, Y, \hat{Z}) \right] \\
    &=
    \bbE_{(Z, Y)\sim\cD_{Z, Y}}
    \big[
        \left.(1 - p_{adv}(x, Y, Z))
        \ell(x, Y, Z) \right.\\
        &\quad+
        p_{adv}(x, Y, Z)
        \ell(x, Y, \hat{Z}(x, Y, Z)
    \big].
\end{align*}

Our goal here is to find a classifier that minimizes
$PR(x)$ without knowing the behavior pattern of
the agents. That is, we are completely unaware of
the dependencies of the performative risk on $x_2$.

To illustrate, consider the simplest case where $d=1$ (i.e. $z, x \in \bbR$).
Specifically, let $\cD_{Z, Y} = \cN(y, 1)$.
The corresponding performative risk and gradients are
illustrated in Figure~\ref{fig:adv_class}(d) and (e).

\begin{figure*}[t!]
  \centering
\includegraphics[width=0.4\textwidth]{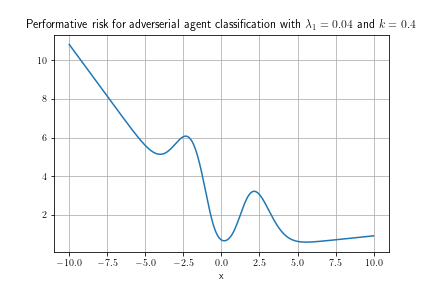}
\includegraphics[width=0.4\textwidth]{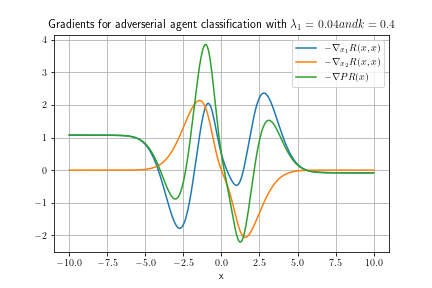}
  \caption{This is a 1-dimensional illustration of the setting described in 
  Section~\ref{sec:adv_class}. When $(z, y) \sim \cN(y, 1)$, $\lambda_1 = 0.04$ and
  $k = 0.4$, the corresponding performative risk is visualized in (d), and the
  corresponding gradients are visualized in (e).}
  \label{fig:adv_class}
\end{figure*}

In both these examples, there are multiple performative risk minimizers and performatively stable points. Performative risk minimization and repeated gradient descent can converge to different steady-state results, and it is of interest which initializations will converge to which equilibria under both dynamics. 
In the sequel, we shall demonstrate how different functions $p(\cdot)$ can lead to different steady-state outcomes, as well as how our theoretical results can provide conditions on $p(\cdot)$ such that we achieve convergence to performative risk minimizers, even when performing repeated approximate risk minimization.


\section{ANALYSIS OF PERFORMATIVE RISK MINIMIZING GRADIENT FLOW}
\label{sec:analysis_prm}

In this section, we consider PRM gradient flow, defined by Equation~\eqref{eq:prm_flow}. We observe that gradient flows provide complete vector fields, and that trajectories will converge to local performative risk minimizers under very mild conditions.

First, we state a proposition guaranteeing that flow is well-defined. The compact sublevel sets ensure that trajectories of Equation~\eqref{eq:prm_flow} remain bounded, which is sufficient to guarantee existence and uniqueness of solutions globally. For proof of the following proposition, we refer the reader to either~{\citet[Section 3.1]{Khalil:2001wj}} or~{\citet[Section 9.3]{Hirsch:2012tx}}.

\begin{proposition}[Existence and uniqueness of gradient flows]

Suppose the performative risk $PR(\cdot)$ is continuously differentiable, and its sublevel sets $\{ x : PR(x) \le c \}$ are compact for every $c \in \mb{R}$. Then for any initial condition $\xnom(0) = x_0$, there exists a unique solution to the differential equation in Equation~\eqref{eq:prm_flow}, defined for all $t \ge 0$.

\end{proposition}

Next, we note that gradient flows have nice properties from the perspective of optimization. Namely: every isolated local minima is locally asymptotically stable, and we can provide sufficient conditions to characterize a subset of the region of convergence.

\begin{proposition}[Convergence of gradient flows]

Suppose the performative risk $PR(\cdot)$ is twice continuously differentiable, and $x^*$ is an isolated local performative risk minimizer. Then $x^*$ is a locally asymptotically stable equilibrium of Equation~\eqref{eq:prm_flow}. 
Furthermore, take any $c$ such that $PR(x^*) \le c$. Let $A \subseteq \{ x : PR(x) \le c \}$ denote the connected component of $\{ x : PR(x) \le c \}$ that contains $x^*$. If $x^*$ is the only local performative minimizer in $A$, then all solutions with initial conditions in $A$ converge to $x^*$.

\end{proposition}

\begin{proof}
Since $x^*$ is an isolated local minimizer and the performative risk is twice continuously differentiable, there exists a neighborhood $U \ni x^*$ such that $\nabla PR(\cdot)$ is non-zero for all $x \neq x^*$. By continuity, there exists some constant $\eps$ such that the connected component of $\{ x : PR(x) \le PR(x^*) + \eps \}$ containing $x^*$ is contained in $U$. Since it is a sublevel set of $PR(\cdot)$ and $\mc{L}_{\fnom}PR(x) < 0$ on its boundary, it is positively invariant. Furthermore, since $\mc{L}_{\fnom}(x) < 0$ for all $x \neq x^*$ on this set, $x^*$ is locally asymptotically stable by standard Lyapunov arguments (see, e.g.~\citet[Section 4]{Khalil:2001wj}).
\end{proof}

The sublevel sets of the performative risk are positively invariant with respect to the PRM gradient flow. Furthermore, because of the continuity of trajectories, each connected component will also be positively invariant. This, in tandem with the fact that trajectories must either converge to a local minima or go off to infinity, also implies the previous proposition.

With minimal assumptions, isolated local performative risk minimizers are all locally attractive in the PRM gradient flow. In Section~\ref{sec:analysis_RGD}, we will view the PRM gradient flow as the nominal dynamics. From this perspective, we analyze the RGD flow as a perturbation from these nominal dynamics. To be able to do any perturbation-based analysis, we will need some stronger conditions on the convergence of the gradient flow associated with performative risk minimization. We note these assumptions here.

\begin{assumption}[Sufficient curvature of the PR]
\label{ass:exist_V}

Fix some isolated local performative risk minimizer $x^*$. 
We assume there exists positive constants $c_1$, $c_2$, $c_3$, $c_4$ and $\delta$ such that the following holds in a neighborhood of $x^*$:
\begin{equation}
\label{eq:v_ineq1}
c_1 |x-x^*|^2 \le PR(x) - PR(x^*) \le c_2 |x-x^*|^2
\end{equation}
\begin{equation}
\label{eq:v_ineq2}
c_3 |x - x^*| - \delta \le | \nabla PR(x) | \le c_4 |x - x^*| + \delta
\end{equation}

We will let $r$ denote the radius of this neighborhood, so the above inequalities are valid on the set $\{ x : |x - x^*| \le r \}$.

\end{assumption}
Assumption~\ref{ass:exist_V} provides conditions on which $V(x) = PR(x) - PR(x^*)$ can be used as a Lyapunov function locally. Next, we provide conditions directly on the loss $\ell(\cdot)$ and the decision-dependent distribution shift $\mc{D}(\cdot)$ which can ensure that Assumption 1 holds. First, we provide sufficient conditions for the bounds in Equation~\eqref{eq:v_ineq1}.

\begin{proposition}[Performative risk bounds]
\label{prop:pr_bnds}
Let $x^*$ be a performative risk minimizer and fix any $x$. If:
\begin{enumerate}
    \item $\ell(\cdot,x)$ is $L_1$ Lipschitz continuous
    \item $\mc{W}_1(\mc{D}(x),\mc{D}(x^*)) \le L_2 |x-x^*|^2$
    \item $\ell(z,\cdot)$ is $m$-strongly convex and $L_3$-smooth for every $z$
\end{enumerate}
Then: $(m/2 - L_1 L_2) |x-x^*|^2 \le PR(x) - PR(x^*) \le (L_1 L_2 + L_3/2) |x-x^*|^2$.
\end{proposition}

\begin{proof}
First, we can break up the performative risk into two parts: $PR(x) - PR(x^*) = R(x,x) - R(x^*,x^*) = [R(x,x) - R(x,x^*)] + [R(x,x^*) - R(x^*,x^*)]$. 
Note that $R(x,x) - R(x,x^*) = \mb{E}_{Z \sim \mc{D}(x)} [\ell(Z,x)] - \mb{E}_{Z \sim \mc{D}(x^*)} [\ell(Z,x)]$. Conditions (1) and (2), along with Kantorovich-Rubenstein duality~\citep{Villani:2003th}, implies this quantity is bounded in absolute value: $|R(x,x) - R(x,x^*)| \le L_1 L_2 |x-x^*|^2$. 
On the other hand, $R(x,x^*) - R(x^*,x^*) = \mb{E}_{Z \sim \mc{D}(x^*)} [\ell(Z,x) - \ell(Z,x^*)]$. By convexity and $L_3$-smoothness, $\ell(z,x) - \ell(z,x^*) \le \langle \nabla_x \ell(z,x^*), x-x^* \rangle + \frac{L_3}{2} |x-x^*|^2$ for any $z$; taking the expectation and noting that $\nabla PR(x^*) = 0$, we have $R(x,x^*) - R(x^*,x^*) \le \frac{L_3}{2} |x-x^*|^2$. In the other direction, using strong convexity and similar arguments, we get: $R(x,x^*) - R(x^*,x^*) \ge \frac{m}{2} |x-x^*|^2$. Combining these results yields the desired results.
\end{proof}
Note that Condition (2) in Proposition~\ref{prop:pr_bnds} is a variation on the typical $\eps$-sensitivity definition. Recall that $\eps$-sensitivity states that for any $x$ and $y$, $\mc{W}_1(\mc{D}(x),\mc{D}(y)) \le \eps|x-y|$~\citep{Perdomo:2020tz}. In contrast, Condition (2) only requires this condition to hold around the point $x^*$, but requires a stricter bound for $x$ close to $x^*$. This bound is also more lax than $\eps$-sensitivity farther away from $x^*$.

Next, we provide sufficient conditions for a bound on the absolute value of the gradient of the performative risk. 

\begin{proposition}[Gradient bounds of the performative risk]
\label{prop:grad_bounds}
Let $x^*$ be a performative risk minimizer and fix any $x$. If:
\begin{enumerate}
    \item $\ell(\cdot,x)$ and $\ell(\cdot, x^*)$ are both $L_1$ Lipschitz continuous
    \item $\ell(z,\cdot)$ is $m$-strongly convex and $L_3$-smooth for every $z$
    \item $\mc{D}(\cdot)$ is $\eps$-sensitive, i.e. $\mc{W}_1(\mc{D}(x),\mc{D}(y)) \le \eps |x-y|$
    \item $\nabla_x \ell(\cdot,x)$ is $L_4$ Lipschitz continuous
\end{enumerate}
Then: $(m- \eps L_4)|x-x^*| - 2 \eps L_1 \le |\nabla PR(x)| \le (L_3 + \eps L_4) |x-x^*| + 2 \eps L_1$.
\end{proposition}

\begin{proof}
Similar to the previous proposition, we break apart this gradient. Note that $\nabla PR(x^*) = 0$, so:
$|\nabla PR(x)| = |\nabla PR(x) - \nabla PR(x^*)| =
|\nabla_{x_1} R(x,x) - \nabla_{x_1} R(x^*,x^*) +
\nabla_{x_2} R(x,x) - \nabla_{x_2} R(x^*,x^*)|$. For the $\nabla_{x_1}$ terms, we have: $m|x-x^*| \le |\nabla_{x_1} R(x,x^*) - \nabla_{x_1}R(x^*,x^*)| \le L_3|x-x^*|$ by standard convexity arguments, and $|\nabla_{x_1} R(x,x) - \nabla_{x_1}R(x,x^*)| \le \eps L_4 |x-x^*|$ by the same Kantorovich-Rubenstein duality argument as the previous proposition. For the $\nabla_{x_2}$ terms, note that the mapping $x_2 \mapsto R(x,x_2)$ is $\eps L_1$ Lipschitz continuous. Thus, $|\nabla_{x_2} R(x,x)| \le \eps L_1$ and similarly $|\nabla_{x_2} R(x^*,x^*)|$. Combining these inequalities yields the desired result.
\end{proof}

Depending on the situation, we may be able to directly verify Assumption~\ref{ass:exist_V}, although, for more complex settings, this is likely to be very difficult. Propositions~\ref{prop:pr_bnds} and~\ref{prop:grad_bounds} provide a set of sufficient conditions for this assumption to hold, but checking the conditions on the decision-dependent distribution shift $\mc{D}(\cdot)$ may be difficult in practice as well. This is one limitation of this current work, and we believe it is an interesting future research direction to identify conditions which are easy to verify, even in settings with limited information about the distribution shift itself.


\section{ANALYSIS OF REPEATED RISK MINIMIZING FLOW}
\label{sec:analysis_RGD}

In the previous section, we consider the PRM gradient flow and showed that the trajectories converge to local performative risk minimizers in very general settings. In this section, we will consider the RGD flow, defined by Equation~\eqref{eq:RGD_flow}. 
The RGD flow is not necessarily a gradient flow, and generally will not inherit the nice properties we saw in Section~\ref{sec:analysis_prm}.

The following theorem provides conditions on the transient response and steady-state behavior of the RGD flow. Prior to $T$, the trajectories converge exponentially quickly. After $T$, we have an ultimate bound that holds.

\begin{theorem}[Ultimate bounds for RGD flow]
\label{th:perturb1}

Fix any isolated performative risk minimizer $x^*$ and suppose the conditions of Assumption~\ref{ass:exist_V} hold. Let $(c_i)_{i=1}^4$ and $\delta$ denote the constants from Assumption~\ref{ass:exist_V} and $r > 0$ denote the radius where the inequalities are valid.

Suppose that there exists positive constants $\eps < c_3^2/c_4$ such that the following holds on $U = \{ x : |x - x^*| \le r \}$:
\begin{equation}
\label{eq:ass_V1}
|\nabla_{x_2}R(x,x)| \le \eps |x-x^*| + \delta
\end{equation}
Additionally, suppose the initial condition satisfies:
\[
|x_0 - x^*| \le \sqrt{\frac{c_1}{c_2}}r
\]
Take any $\theta \in (0,1)$ such that:
\[
\delta \le \sqrt{\frac{c_1}{c_2}} 
\frac{(1 - \theta) r (c_3^2/c_4 - \eps)}{c_4 + 2c_3 + \epsilon}
\]
Then, there exists a $T \ge 0$ such that:
\begin{itemize}
\item For all $t \le T$:
\[
\begin{aligned}
&|\varphi_{\fpert}(t;x_0) - x^*| \le \\
&\qquad\sqrt{\frac{c_2}{c_1}} \exp(-t\theta (c_3^2 - c_4\eps)/2c_2) |x_0 - x^*|
\end{aligned}
\]
\item For all $t \ge T$:
\[
\begin{aligned}
    &|\varphi_{\fpert}(t;x_0) - x^*| 
    \le \\
    &\qquad\sqrt{\frac{c_2}{c_1}} 
    \max
    \left\{
    \frac{\delta (c_4 + 2c_3 + \epsilon)}{ (1-\theta) (c_3^2 - c_4 \eps)},
    \frac{\delta}{c_3}
    \right\}.
\end{aligned}
\]
\end{itemize}

\end{theorem}

\begin{proof}
See Appendix~\ref{app:proof_perturb1}.
\end{proof}

Note that, in the special case where $\delta = \lambda = 0$, we have that the RGD flow converges exponentially quickly to $x^*$ locally. Similarly, in the special case where Assumption~\ref{ass:exist_V} holds everywhere (i.e. $r = \infty$), then there is only one minimizer $x^*$, and all initial conditions converge to a neighborhood of $x^*$ exponentially fast. In addition, if every condition in Assumption 1 holds with equality, the bounds in Theorem 1 also hold with equality. Please see Appendix A for an example where this occurs.

Additionally, note that locally performatively stable points are equilibria of the RGD flow. This result provides constraints on where performatively stable points can be located. Consider the special case where Assumption~\ref{ass:exist_V} holds globally (i.e. $r = \infty$) and, consequently, there exists only one minimizer $x^*$. In this special case, Theorem~\ref{th:perturb1} shows that all performatively stable points must be close to $x^*$. The phenomena that, under certain conditions, performatively stable points are near performative risk minimizers, was first noted in~\citet{Perdomo:2020tz}. Our results here provide another set of conditions under which the same result holds.

Furthermore, in the presence of Propisition 3 and 4, Theorem 1 can be restated in terms of $\ell(\cdot)$ and $\cD(\cdot)$ as follows:

\begin{corollary}
    Let $x^*$ be a performative risk minimizer and fix any $x$. If:
    \begin{enumerate}
    \item $\ell(\cdot,x)$ and $\ell(\cdot, x^*)$ are both $L_1$ Lipschitz continuous
    \item $\ell(z,\cdot)$ is $m$-strongly convex and $L_3$-smooth for every $z$
    \item $\mc{D}(\cdot)$ is $\eps$-sensitive, i.e. $\mc{W}_1(\mc{D}(x),\mc{D}(y)) \le \eps |x-y|$
    \item $\mc{W}_1(\mc{D}(x),\mc{D}(x^*)) \le L_2 |x-x^*|^2$
    \item $\nabla_x \ell(\cdot,x)$ is $L_4$ Lipschitz continuous
\end{enumerate}
Suppose the initial condition satisfies:
\[
|x_0 - x^*| \le \sqrt{\frac{m - 2L_1L_2}{L_3 + 2L_1L_2}}r
\]
Then, there exists a $T \ge 0$ such that:
\begin{itemize}
\item For all $t \le T$:
\[
|\varphi_{\fpert}(t;x_0) - x^*| \le 
\]
\[
\sqrt{\frac{L_3 + 2L_1L_2}{m - 2L_1L_2}} \exp\left(\frac{-t\theta (m-\epsilon L_4)^2}{L_3 + 2L_1L_2}\right) |x_0 - x^*|
\]
\item For all $t \ge T$:
\[
\begin{aligned}
    &|\varphi_{\fpert}(t;x_0) - x^*| \le 
\sqrt{\frac{L_3 + 2L_1L_2}{m - 2L_1L_2}} \\
&\qquad\cdot\max
\left\{
    \frac{4\epsilon L_1 (L_3 + 2m - \epsilon L_4)}{ (m-\epsilon L_4)^2},
    \frac{2 \epsilon L_1}{m - \epsilon L_4}
\right\}.
\end{aligned}
\]
\end{itemize}
\end{corollary}
\begin{proof}
This follows immediately by combining Theorem~\ref{th:perturb1} with Propositions~\ref{prop:pr_bnds} and~\ref{prop:grad_bounds}.
\end{proof}


\subsection{Performative alignment}

From the previous analysis, we also identify conditions on the directions of the performative perturbations that are sufficient to show the convergence of Equation~\eqref{eq:RGD_flow}, the RGD flow, to performative risk minimizers.

\begin{theorem}[Performative alignment]
\label{th:perf_align}
Suppose $x^*$ is a isolated local performative risk minimizer and the following holds for all $x$ in a neighborhood of $x^*$:
\begin{equation}
\label{eq:perf_align}
|\nabla_{x_2}R(x,x)|^2 \le \langle -\nabla_{x_1}R(x,x), \nabla_{x_2}R(x,x) \rangle
\end{equation}
Then $x^*$ is a locally asymptotically stable equilibrium point of the RGD flow, given by Equation~\eqref{eq:RGD_flow}. 
Note that this does \textbf{not} require Assumption~\ref{ass:exist_V}.
\end{theorem}

\begin{proof}
Let $V(x) = PR(x) - PR(x^*)$. 
Since $x^*$ is a locally asymptotically equilibria of the PRM flow, we have: $V(x^*) = 0$, $V(x) > 0$ for $x \neq 0$, and $\mc{L}_{\fnom} V(x) < 0$ for $x \neq 0$. The performative alignment condition ensures that $\mc{L}_{\fnom + g} V(x) < 0$ as well, and the desired result follows.
\end{proof}

We refer to Equation~\eqref{eq:perf_align} as the \textbf{performative alignment} condition. This condition states that the performative perturbation never increases the performative risk, and the convergence of performative risk minimization is sufficient to guarantee convergence of repeated risk minimization. In other words, the perturbation is pointing in the correct direction to ensure that $PR(\cdot) - PR(x^*)$ can still act as a Lyapunov function.

Another perspective on performative alignment is to consider the performative risk as a bilinear form whose arguments are parameterized by $x$. In particular, consider the decoupled performative risk $R(\cdot,\cdot)$. Let $\ell_{x} := \ell(\cdot, x)$ and let $\mu_x$ denote the probability distribution associated with $\mc{D}(x)$. Then, we can write $R(x_1,x_2) = \langle \mu_{x_2}, \ell_{x_1} \rangle$. From this perspective, $R(\cdot,\cdot)$ is a bilinear form in $\ell_x$ and $\mu_x$. As such, the performative alignment condition becomes a condition on the way in which $\ell$ and $\mu$ are \textit{parameterized} by $x$.

In Appendix~\ref{sec:perf_align_ex}, we apply Theorem~\ref{th:perf_align} to the example outlined in Section~\ref{sec:simple_ex}. It provides insight into one of the ways to use Theorem~\ref{th:perf_align}: when we fix a loss $\ell(\cdot)$, we can view the performative alignment condition as specifying a class of decision-dependent distribution shifts which do not hamper the convergence of RGD to performative risk minimizers.




\section{CLOSING REMARKS}
\label{sec:conclusion}

In this paper, we analyzed the problem of performative prediction in settings where multiple isolated equilibria may be of interest. We analyzed the gradient flow of performative risk minimization, and identified regions of attraction for various equilibria. We viewed repeated gradient descent flow as a perturbation of the PRM gradient flow. In particular, we used a Lyapunov function for the PRM gradient flow to analyze the trajectories of the RGD flow. We found conditions on which RGD flow will converge to the local PRM minimizers, and conditions on which they will converge to a neighborhood of PRM minimizers. 

These results provide a method to analyze the regions of attraction for various equilibria under repeated risk minimization. In real-world settings with decision-dependent distributions, we expect many situations where the initialization may have a significant outcome on the trajectories and final outcomes. 


\subsubsection*{Acknowledgements}
Lillian J. Ratliff is supported by NSF CAREER Award No.1844729.

\bibliography{main}

\appendix
\onecolumn

\section{PROOF OF THEOREM~\ref{th:perturb1}}
\label{app:proof_perturb1}

Let $V(x) = R(x,x) - R(x^*,x^*)$. Note that $V(x) \ge 0$ on $U = \{ x : |x - x^*| \le r \}$ and $V(x) = 0$ if and only if $x = x^*$. Furthermore, note that $\frac{\partial V}{\partial x}(x) = [\nabla_{x_1}R(x,x) + \nabla_{x_2}R(x,x)]^\T$.

Consider the function $t \mapsto V(\varphi_{\fpert}(t;x_0))$ and its time derivative. Also, let $\xpert(t) = \varphi_{\fpert}(t;x_0)$. 
When $|\xpert-x^*| \ge \delta/c_3$, taking the derivative along trajectories of the repeated risk minimization flow and using Equations~\eqref{eq:v_ineq2} and~\eqref{eq:ass_V1}:
\[
\begin{aligned}
    \mc{L}_{\fnom + g}V 
    &=
    \frac{\partial V}{\partial x} (\fnom(x) + g) 
    =
    -|\nabla_{x_1}R + \nabla_{x_2}R|^2 + \langle \nabla_{x_1}R + \nabla_{x_2}R, \nabla_{x_2}R \rangle \\
    &\le
    -\left( c_3 |\xpert-x^*|-\delta \right)^2
    +
    \left( c_4 |\xpert-x^*|+\delta \right) |\nabla_{x_2}R| \\
    &\le
    -\left( c_3 |\xpert-x^*|-\delta \right)^2
    +
    \left( c_4 |\xpert-x^*|+\delta \right) (\epsilon|\xpert-x^*|+\delta)\\
    &=
    -\left(c_4\epsilon-c_3^2\right)|\xpert-x^*|^2
    +\left(c_4\delta + 2c_3\delta+\delta\epsilon\right)|\xpert-x^*|
\end{aligned}
\]
These inequalities are valid so long as $\xpert(t)$ stays within $U$, which we will ensure later in the proof. Note that $\eps$ is sufficiently small (by assumption) to ensure that $-c_3^2 + c_4 \eps < 0$.

Let $\alpha := c_3^2 - c_4 \eps > 0$. Take any $\theta \in (0,1)$ and note that:
\[
\mc{L}_{\fnom + g}V(\xpert) \le - \theta \alpha | \xpert - x^* |^2 - (1 - \theta) \alpha | \xpert - x^* |^2 +\left(c_4\delta + 2c_3\delta+\delta\epsilon\right)|\xpert-x^*|
\]
Now, let 
\[
    \mu(\theta)
    :=
    \max
    \left\{
    \frac
    {
        \left(c_4\delta + 2c_3\delta+\delta\epsilon\right)
    }
    {\alpha(1 - \theta)},
    \frac{\delta}{c_3}
    \right\}.
\]
If $|\xpert - x^*| \ge \mu(\theta)$, then:
\[
- \theta \alpha | \xpert - x^* |^2 - (1 - \theta) \alpha | \xpert - x^* |^2 +\left(c_4\delta + 2c_3\delta+\delta\epsilon\right)|\xpert-x^*| \le 0,
\]
and thus
\[
\mc{L}_{\fnom + g}V(\xpert) \le - \theta \alpha | \xpert - x^* |^2.
\]
Trajectories of Equation~\eqref{eq:RGD_flow} has two stages: a transient due to its initial condition, and then an ultimate bound due to the perturbation. Let $T(\theta) = \inf~\{ t \ge 0 : |\xpert(t) - x^*| \le \mu(\theta) \}$. Prior to $T(\theta)$, we have:
\[
\frac{d}{dt} V(\xpert(t)) \le - \theta \alpha |\xpert(t) - x^*|^2 \le - \frac{\theta \alpha}{c_2} V(\xpert(t))
\]
The latter follows from Equation~\eqref{eq:v_ineq1}. 
By the comparison principle (see, e.g.~\citep[Lemma 3.4]{Khalil:2001wj}), we have $V(\xpert(t)) \le \exp(-t\theta \alpha / c_2) V(x_0)$. Again using Equation~\eqref{eq:v_ineq1}, this yields the following inequality, valid for all $t \le T(\theta)$:
\[
|\xpert(t) - x^*| \le \sqrt{\frac{c_2}{c_1}} \exp(-t\theta \alpha/2c_2) |x_0 - x^*|
\]
Note that this inequality also provides an upper bound on $T(\theta)$. Additionally, note that this implies the bound $|\xpert(t) - x^*| \le r$, by our assumption on the initial condition. Prior to $T(\theta)$, our trajectory stays in $U$, where our inequalities are valid.

At time $T(\theta)$, we have $|\xpert(t) - x^*| \le \mu(\theta)$. Note that this inequality implies $V(\xpert(t)) \le c_2 \mu^2(\theta)$. Since $\mc{L}_{\fnom + g}V < 0$ on the boundary of $\Omega(\theta) := \{ x : V(x) \le c_2 \mu^2(\theta) \}$, we have that $\Omega(\theta)$ is a positively invariant set. So, for all $t \ge T(\theta)$, we have $\xpert(t) \in \Omega(\theta)$. Using Equation~\eqref{eq:v_ineq1}, we have the following for all $t \ge T(\theta)$:
\[
|\xpert(t) - x^*| \le 
\sqrt{\frac{c_2}{c_1}} \mu(\theta) 
\]
The condition on $\theta$ ensures that this quantity is bounded by $r$, and the trajectory stays in $U$ for $t \ge T(\theta)$. 
This proves our desired result.

Additionally, we can show that this bound is tight by considering the following example. 
Suppose $\cD(x_2)$ is the point mass distribution (i.e. $p(z) = \delta(z - x_2)$) and $l(z, x_1) = 1/2 |z|^2 + 1/2|x_1|^2$. Then the performative risk is given by $R(x_1, x_2) = 1/2 |x_1|^2 + 1/2|x_2|^2$. It follows that $x^* 0$ is the performative risk minimizer. Following the arguments in Appendix A, one would find that the dynamics of $V(x) = R(x, x) - R(x^*, x^*)$ follows $\frac{d}{dt} V(x(t)) = -2 |x(t) - x^*|^2 = -2 V(x(t))$, which yields 
    $|x(t) - x^*| = \exp(-2t)|x_0 - x^*|$.

\section{NUMERICAL EXAMPLES}
\label{sec:num_results}

\begin{figure*}[t!]
  \centering
\includegraphics[width=0.325\textwidth]{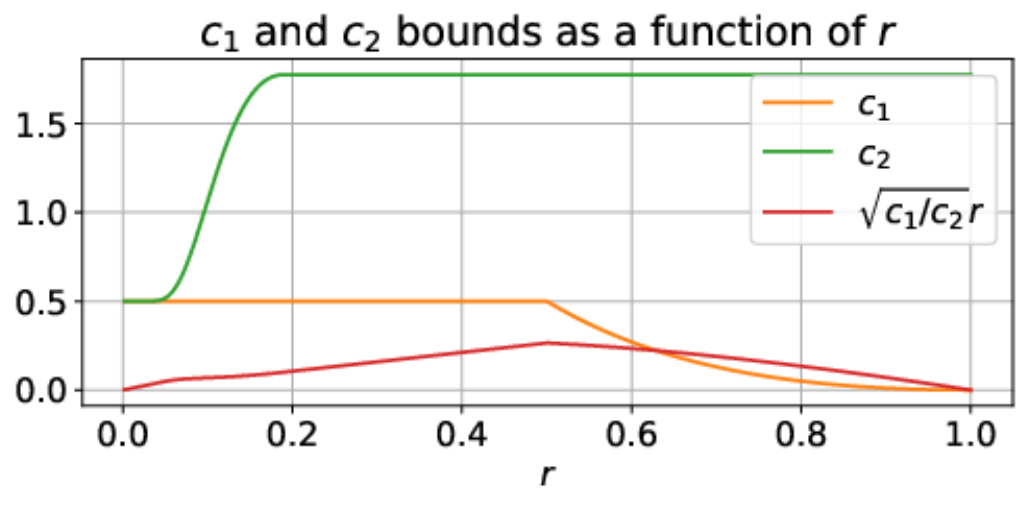}
\includegraphics[width=0.315\textwidth]{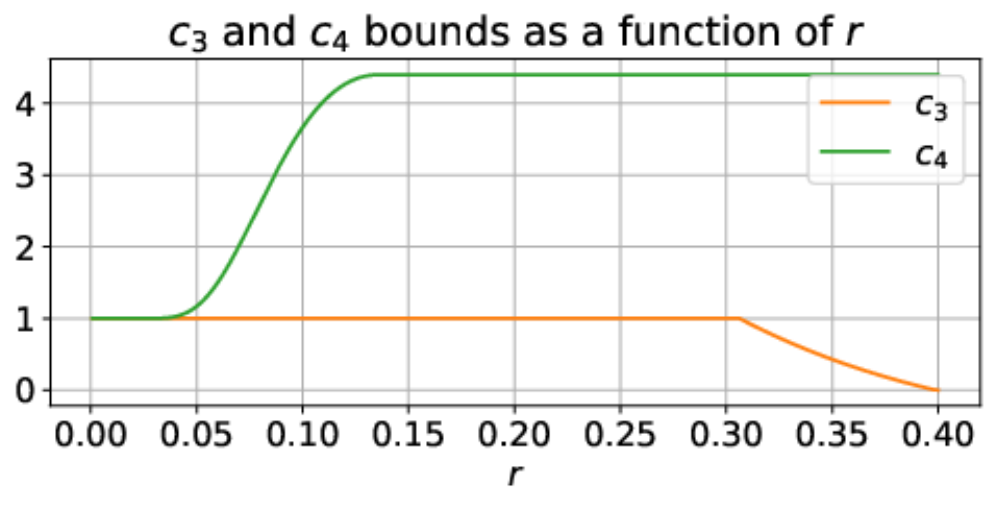}
\includegraphics[width=0.26\textwidth]{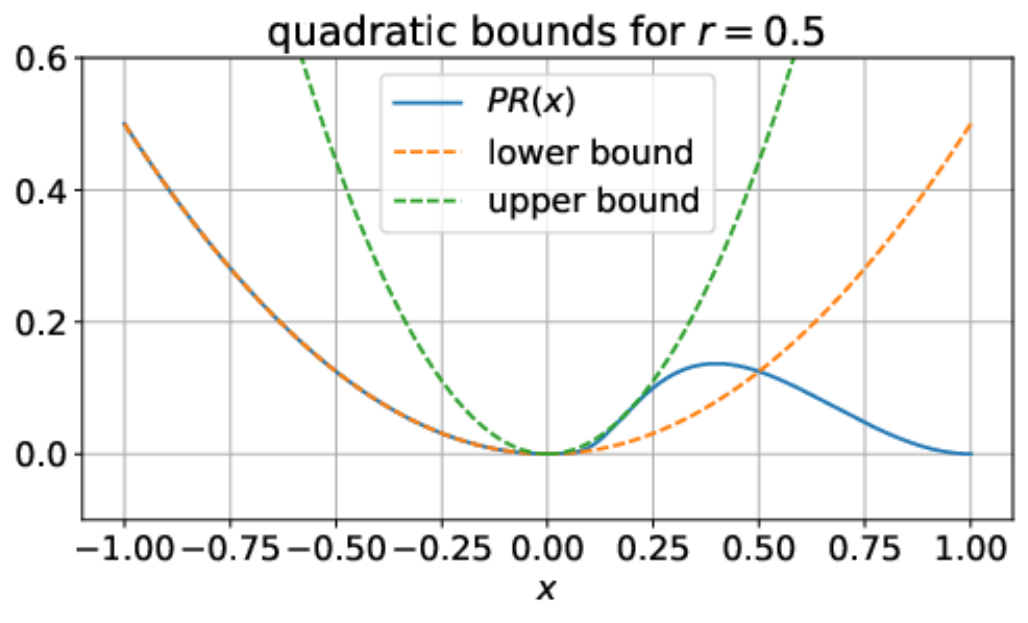}
  \caption{We verify that the performative risk bounds in Assumption~\ref{ass:exist_V} are satisfied in the example discussed in Section~\ref{sec:example}. (a) As a function of $r$ (the radius of the domain where the inequalities hold), we show the tightest constants $c_1$ and $c_2$ for the bound. We also plot $\sqrt{c_1/c_2}r$, which is the radius of a neighborhood of $x=0$ to which Theorem~\ref{th:perturb1} can be applied. (b) As a function of $r$, we show the tightest constants for $c_3$ and $c_4$. (c) Choosing the $c_1$ and $c_2$ constants for $r = 0.5$, we visualize how the quadratic bounds hold for the performative risk locally.}
  \label{fig:simp_ex_demo}
\end{figure*}
In this section, we revisit the models introduced in Section~\ref{sec:example}. We demonstrate how the results of Sections~\ref{sec:analysis_prm} and~\ref{sec:analysis_RGD} can be applied. First, we show that the example satisfies Assumption~\ref{ass:exist_V} and we calculate its corresponding constants. Second, we apply Theorem~\ref{th:perturb1} and show the theoretical convergence rates match simulated trajectories. 
Finally, we also apply Theorem~\ref{th:perf_align} to the example from Section~\ref{sec:simple_ex} and characterize the class of distribution shifts satisfy the performative alignment condition.

\subsection{Checking the curvature of the performative risk and region of convergence}

Recall the example from Section~\ref{sec:simple_ex}, where $x$ was a scalar, the loss function was the squared error, and the decision-dependent distribution was a Bernoulli random variable whose distribution was determined by $p(\cdot)$. In this section, we consider the specific decision-dependent distribution shift $p = \varphi$, which is defined in Equation~\eqref{eq:varphi_def}.

When we consider this example, we can see that the bounds on Assumption~\ref{ass:exist_V} cannot hold globally, which matches our previous observation that there are multiple isolated performative risk minimizers. However, these bounds may hold locally: we can view the constants $(c_i)_{i=1}^4$ from Assumption~\ref{ass:exist_V} as a function of the size of the domain $r$.

For concreteness, let us focus on the equilibrium point $x = 0$. Recall that Assumption~\ref{ass:exist_V} must hold locally, on the domain $\{ x : |x-x^*| \le r\}$. As we increase $r$, the constants will worsen; we visualize this in Figure~\ref{fig:simp_ex_demo}(a)--(b). Note that these bounds only have to hold locally around the equilibria, as visualized in Figure~\ref{fig:simp_ex_demo}(c). Furthermore, the gradient bounds in Assumption~\ref{ass:exist_V} cannot hold beyond $r > 0.40$, since $\nabla PR(x) = 0$ at that point.

Recall that the convergence results of Theorem~\ref{th:perturb1} can only apply to all initial conditions satisfying $|x_0 - x^*| < \sqrt{c_1/c_2}r$; we visualize this as well in Figure~\ref{fig:simp_ex_demo}(a). 
On the set $(0,0.40]$, we can see the quantity $\sqrt{c_1/c_2}r$ is the largest at $r = 0.4$, with constants $c_1 = 0.50$ and $c_2 = 1.78$. 
Thus, around the equilibrium $x = 0$, the theorem can be applied to all points in the set $\{ x : |x| \le 0.21 \}$, with $\delta = 0$. Thus, our theorem shows that all points in this neighborhood of $x = 0$ will converge. This under-approximates the true region of attraction, which we numerically saw to be $\{ x : x < 0.23 \}$.

\subsection{Performative alignment with squared error and Bernoulli distributions}
\label{sec:perf_align_ex}

We again consider the example from Section~\ref{sec:simple_ex}. However, in this section, we consider a general decision-dependent distribution shift $p(\cdot)$. 
We suppose that $p(0) = 0$ and $p(1) = 1$, so we have two performative risk minimizers as in our previous example. We have $\nabla_{x_1}R(x,x) = x - p(x)$ and $\nabla_{x_2}R(x,x) = (1/2 - x) p'(x)$. 
The performative alignment condition becomes:
\begin{equation}
    \label{eq:perf_align_ex}
    |1/2 - x|^2 |p'(x)|^2 \le (p(x)-x)(1/2 - x)p'(x)
\end{equation}
Theorem~\ref{th:perf_align} states that if this condition holds for all $x \in (0,c)$, then any initial conditions $x_0 \in (0,c)$ will converge to $x = 0$. Similarly, if this condition holds for all $x \in (c,1)$, then all initial conditions in $(c,1)$ will converge to $x = 1$. Theorem~\ref{th:perf_align} also implies that this condition cannot be satisfied for all $x \in (0,1)$, as then these initial conditions would converge to \textit{both} $x = 0$ and $x = 1$.

If we suppose that $p(\cdot)$ is monotonic on $(0,1)$, i.e. $p'(x) \ge 0$, we can also interpret the performative alignment condition as follows. For $x \in (1/2,1)$, the performative alignment condition becomes $p(x) - x \ge (1/2 - x)p'(x)$. In this regime, $(1/2 - x)p'(x) \le 0$. In this setting, if $p(x) - x$ is too negative, the RGD flow will push $x$ away from the nearby minimizer $x = 1$. Similarly, for $x \in (0,1/2)$, the condition becomes $p(x) - x \le (1/2 - x)p'(x)$. In this regime, $(1/2 - x)p'(x) \ge 0$, and the condition states that $p(x) - x$ cannot be too large, or the RGD flow will push $x$ away from the minimizer $x = 0$.

In this section, we used Theorem~\ref{th:perf_align} to identify conditions on the decision-dependent distribution shift $p(\cdot)$ which ensure that the performative risk does not increase even when the dynamics follow repeated gradient descent.
For this example, the condition is that $p$ satisfies Equation~\eqref{eq:perf_align_ex} for all $x \in (0,c)$. 
More generally, the performative alignment condition allow us to specify a class of distribution shifts which behave well with respect to performative risk minimization.


\end{document}